%% file: main.tex
\documentclass{article} %
\usepackage[accepted]{icml2021}

\usepackage{url}
\usepackage[utf8]{inputenc}
\usepackage{cite}
\usepackage[right]{lineno}
\usepackage{microtype}
\DisableLigatures[f]{encoding = *, family = * }
\definecolor{Gray}{gray}{.25}
\usepackage{graphicx}
\usepackage{sidecap}
\usepackage{wrapfig}
\usepackage{eso-pic}
\usepackage[fulladjust]{marginnote}
\reversemarginpar
\usepackage{graphicx} 
\usepackage{subcaption}
\usepackage{slashbox}
\usepackage{enumerate}

\usepackage{natbib}
\usepackage{amsmath}
\usepackage{amsthm}
\usepackage{amssymb}
\usepackage{tikz}
\usetikzlibrary{arrows}
\allowdisplaybreaks

\usepackage{xcolor}
\usepackage{hyperref}
\definecolor{violetblue}{rgb}{0.54, 0.17, 0.89}

\hypersetup{
  colorlinks=true,
  linkcolor=magenta,
  citecolor=violetblue,
  urlcolor=magenta,
  bookmarksnumbered=true
}

\usepackage{mathrsfs}
\usepackage{bbm}

\usepackage{thmtools}
\usepackage{thm-restate}
\PassOptionsToPackage{numbers,compress,sort,round}{natbib}

\usepackage[utf8]{inputenc} %
\usepackage[T1]{fontenc}    %
\usepackage{booktabs}       %
\usepackage{amsfonts}       %
\usepackage{nicefrac}       %
\usepackage{microtype}      %
\usepackage{caption}
\usepackage{amssymb}
\usepackage{graphicx}
\usepackage{color}
\usepackage{subcaption}
\usepackage{amsmath}
\usepackage{amsthm}
\usepackage{pifont}
\usepackage{tabularx}
\usepackage{fancyvrb}
\usepackage{comment}
\usepackage{float}
\usepackage{wrapfig}
\usepackage{mathtools}
\usepackage{thmtools,thm-restate}
\usepackage{paralist} 
\usepackage{multirow}
\usepackage{scrextend}
\usepackage{enumitem}
\usepackage[noorphans]{quoting}
\usepackage{bm}
\usepackage{bbm}
\usepackage{thm-restate}
\usepackage{tabularx}
\usepackage{xspace}
\usepackage{times}
\usepackage{tabu}
\usepackage{icomma}
\usepackage{listings}
\usepackage{amssymb}%
\usepackage{pifont}%
\lstdefinestyle{lfonts}{
  basicstyle   = \footnotesize\ttfamily,
  stringstyle  = \color{purple},
  keywordstyle = \color{blue!60!black}\bfseries,
  commentstyle = \color{olive}\scshape,
}
\lstdefinestyle{lnumbers}{
  numbers     = left,
  numberstyle = \tiny,
  numbersep   = 1em,
  firstnumber = 1,
  stepnumber  = 1,
}
\lstdefinestyle{llayout}{
  breaklines       = true,
  tabsize          = 2,
  columns          = flexible,
}
\lstdefinestyle{lgeometry}{
  xleftmargin      = 20pt,
  xrightmargin     = 0pt,
  frame            = tb,
  framesep         = \fboxsep,
  framexleftmargin = 20pt,
}
\lstdefinestyle{lgeneral}{
  style = lfonts,
  style = lnumbers,
  style = llayout,
  style = lgeometry,
}
\lstdefinestyle{python}{
    language = {Python},
    style    = lgeneral,
}

\newcommand{\cmark}{\ding{51}}%
\newcommand{\xmark}{\ding{55}}%

\usepackage{courier}
\numberwithin{equation}{section}
\usepackage{chngcntr}
\counterwithout{equation}{section} 
\theoremstyle{definition}
\newtheorem{theorem}{Theorem}

\newtheorem{proposition}[theorem]{Proposition}
\newtheorem{lemma}[theorem]{Lemma}
\newtheorem{claim}[theorem]{Claim}
\theoremstyle{definition}

\theoremstyle{remark}

\newcommand{\dist}{\mathcal{D}}

\newcommand{\xxspace}{\mathcal{X}}
\newcommand{\yyspace}{\mathcal{Y}}
\newcommand{\zzspace}{\mathcal{Z}}

\newcommand{\RR}{\mathbb{R}}

\newcommand{\Exp}{\mathbb{E}}
\newcommand{\HH}{\mathcal{H}}
\newcommand{\FF}{\mathcal{F}}
\newcommand{\Adv}{\text{Adv}}

\DeclareMathOperator*{\argmin}{arg\,min}

\newcommand{\tr}{\text{Tr}}

\newcommand{\defeq}{\vcentcolon=}

\newcommand{\eps}{\varepsilon}

\newcommand{\ind}{\mathbb{I}}
\DeclarePairedDelimiterX{\inp}[2]{\langle}{\rangle}{#1, #2}

\newcommand{\dtv}{d_{\text{TV}}}

\icmltitlerunning{Information Obfuscation of Graph Neural Networks}

\author{Peiyuan Liao\thanks{Equal contribution}\thanks{Carnegie Mellon University. Email: \texttt{peiyuanl@andrew.cmu.edu}}
\and 
Han Zhao\printfnsymbol{1}\thanks{University of Illinois at Urbana-Champaign. Email: \texttt{hanzhao@illinois.edu}}
\and
Keyulu Xu\printfnsymbol{1}\thanks{Massachusetts Institute of Technology. Email: \texttt{keyulu@mit.edu}}
\and
Tommi Jaakkola\thanks{Massachusetts Institute of Technology. Email: \texttt{tommi@csail.mit.edu}}
\and
Geoffrey Gordon\thanks{Carnegie Mellon University. Email: \texttt{ggordon@cs.cmu.edu}}
\and
Stefanie Jegelka\thanks{Massachusetts Institute of Technology. Email:  \texttt{stefje@mit.edu}}
\and
Ruslan Salakhutdinov\thanks{Carnegie Mellon University. Email:   \texttt{rsalakhu@cs.cmu.edu}}
}

\begin{document}

\twocolumn[
\icmltitle{Information Obfuscation of Graph Neural Networks}

\icmlsetsymbol{equal}{*}

\begin{icmlauthorlist}
\icmlauthor{Peiyuan Liao}{equal,cmu}
\icmlauthor{Han Zhao}{equal,uiuc}
\icmlauthor{Keyulu Xu}{equal,mit}\\
\icmlauthor{Tommi Jaakkola}{mit}
\icmlauthor{Geoffrey Gordon}{cmu}
\icmlauthor{Stefanie Jegelka}{mit}
\icmlauthor{Ruslan Salakhutdinov}{cmu}
\end{icmlauthorlist}

\icmlaffiliation{mit}{Massachusetts Institute of Technology (MIT)}
\icmlaffiliation{cmu}{Carnegie Mellon University}
\icmlaffiliation{uiuc}{University of Illinois at Urbana-Champaign}

\icmlcorrespondingauthor{Peiyuan Liao}{peiyuanl@andrew.cmu.edu}
\icmlcorrespondingauthor{Han Zhao}{hanzhao@illinois.edu}
\icmlcorrespondingauthor{Keyulu Xu}{keyulu@mit.edu}

\icmlkeywords{Machine Learning, ICML}

\vskip 0.3in
]

\printAffiliationsAndNotice{\icmlEqualContribution} %

\begin{abstract}
While the advent of Graph Neural Networks (GNNs) has greatly improved node and graph representation learning in many applications, the neighborhood aggregation scheme exposes additional vulnerabilities to adversaries seeking to extract node-level information about sensitive attributes. In this paper, we study the problem of protecting sensitive attributes by information obfuscation when learning with graph structured data. We propose a framework to locally filter out pre-determined sensitive attributes via adversarial training with the total variation and the Wasserstein distance. Our method creates a strong defense against inference attacks, while only suffering small loss in task performance. Theoretically, we analyze the effectiveness of our framework against a worst-case adversary, and characterize an inherent trade-off between maximizing predictive accuracy and minimizing information leakage. Experiments across multiple datasets from recommender systems, knowledge graphs and quantum chemistry demonstrate that the proposed approach provides a robust defense across various graph structures and tasks, while producing competitive GNN encoders for downstream tasks.  
\end{abstract}

\input{intro.tex}

\input{preliminary.tex}

\input{algo.tex}

\input{experiment.tex}

\vspace*{-0.5em}
\input{conclusion.tex}

\newpage
\bibliography{main}
\bibliographystyle{icml}

\newpage
\appendix
\input{appendix.tex}

\end{document}

%% file: intro.tex
\section{Introduction}
\label{sec:intro}
Graph neural networks (GNNs) have brought about performance gains in various tasks involving graph-structured data~\citep{scarselli2009graph, xu2018how}. A typical example includes movie recommendation on social networks~\citep{ying2018graph}. %
Ideally, the recommender system makes a recommendation not just based on the description of an end user herself, but also those of her close friends in the social network. By taking the structured information of friendship in social network into consideration, more accurate prediction is often achieved~\citep{xu2018representation, hamilton2017inductive}. However, with better utility comes more vulnerability to potential information leakage. For example, to gain sensitive information about a specific user in the network, malicious adversaries could try to infer sensitive information not just only based on the information of the user of interest, but also information of her friends in the network. Such scenarios are increasingly ubiquitous with the rapid growth of users in common social network platforms, especially in the distributed/federated setting where data are transmitted from local nodes to centralized servers and the goal of malicious adversaries is to reveal users' sensitive data by eavesdropping during the transmission process. Hence, the above problem poses the following challenge: 
\begin{quoting}
\itshape
\vspace*{-0.2em}
    How could we protect sensitive information of users in the network from malicious inference attacks while maintaining the utility of service? Furthermore, can we quantify the potential trade-off between these two goals?
\vspace*{-0.1em}
\end{quoting}
\begin{figure*}[tb]
    \centering
    \begin{subfigure}[htb]{\linewidth}
        \centering
        \includegraphics[width=0.8\textwidth]{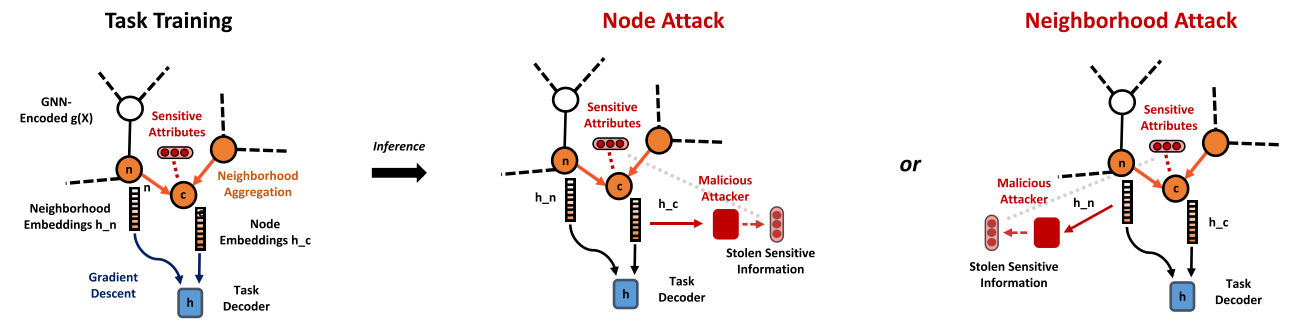}
        \vspace{-0.5em}
        \caption{Information Leakage in Client-Server Threat Model}
    \end{subfigure}%
    
    \begin{subfigure}[htb]{\linewidth}
        \centering
        \includegraphics[width=0.8\textwidth]{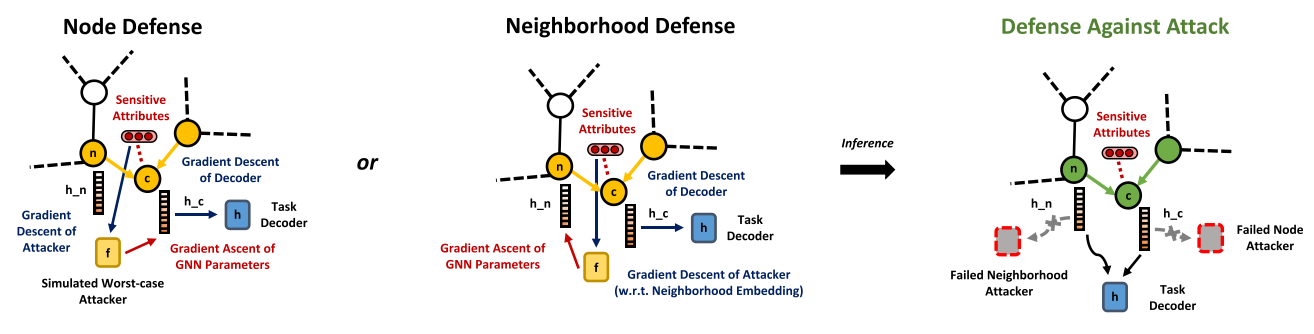}
        \vspace{-0.5em}
        \caption{Information Obfuscation with GAL}
    \end{subfigure}
    
    \caption{\textbf{Graph Adversarial Networks (GAL).} Above represents a common challenge of inference attack present in graph learning algorithms (a) and our proposed solution (b). For the figure above, starting from left to right, we present a way for adversaries to obtain sensitive information from graph neural networks: after training via neighborhood aggregation (1), by running neural networks on produced embedding of nodes (2) and their neighbors (3) in graphs, they are able to recover sensitive information about them. GAL, on the other hand, defends node and neighborhood inference attacks via a minimax game between the task decoder (blue) and a simulated worst-case adversary (yellow) on both the embedding (descent) and the attributes (ascent). Malicious adversaries will have difficulties extracting sensitive attributes at inference time from GNN embeddings trained with our framework. 
    }
    \label{mainvis}
\vspace*{-1em}
\end{figure*}
In this paper, we provide answers to both questions. We propose a simple yet effective algorithm to achieve the first goal through adversarial training of GNNs, a general framework which we term as \textbf{G}raph \textbf{A}dversaria\textbf{L} Networks (GAL). In a nutshell, the proposed algorithm learns node representations in a graph by simultaneously preserving rich information about the target task and filtering information from the representations that is related to the sensitive attribute via a minimax game (Figure~\ref{mainvis}). While the minimax formulation is not new and has been applied broadly in image generation, domain adaptation and robustness~\citep{goh2010distributionally,goodfellow2014generative,madry2017towards,zhao2018adversarial}, we are the first to formulate this problem on graphs for information obfuscation, and to demonstrate that minimax optimization is effective for GNNs in these settings, both theoretically and empirically.

At a high level, for the purpose of information obfuscation of certain sensitive attributes, we propose a method to locally filter out these attributes by learning GNN encoders that could confuse a strong malicious adversary. In particular, the learned GNN encoders will generate node representations from which even a worst-case adversary cannot reveal the sensitive attributes. Technically, we use the total variation and the dual formulation of the Wasserstein distance to form a minimax game between the GNN encoder and the adversary, and train it until convergence. We provide theoretical guarantees for our algorithm, and quantify the potential trade-off between GNN predictive accuracy and information leakage. First, we prove a \emph{lower bound} for the inference error over the sensitive attribute that any worst-case adversary has to incur under our algorithm. Second, we quantify how much one has to pay in terms of predictive accuracy for information obfuscation. Specifically, we prove that the loss in terms of predictive accuracy is proportional to how the target task is correlated with the sensitive attribute in input node features. 

Empirically, we corroborate our theory and the effectiveness of the proposed framework on 6 graph benchmark datasets. We show that our framework can both train a \emph{competitive} GNN encoder and perform \emph{effective information obfuscation}. For instance, our algorithm successfully decreases the AUC of a gender attacker by 10\% on the Movielens dataset while only suffering 3\% in task performance. Furthermore, our framework is robust against a new set of attacks we term ``neighborhood attacks'' or ``$n$-hop attacks'', where the adversary can infer node-level sensitive attributes from embeddings of $1$- and $n$-distant neighbors. We verify that in these new settings, our algorithm remains effective. Finally, our theoretical bounds on the trade-off between accuracy and defense also agree with experimental results.~\footnote{Code at: \url{https://github.com/liaopeiyuan/GAL}}

In summary, we formulate the information obfuscation problem on GNNs. In this new setting, we show that GNNs trained with GAL achieve both good predictive power and information obfuscation, theoretically (Theorem~\ref{thm:guarantee}) and empirically (Table~\ref{task}). Our theory also quantifies the trade-off between accuracy and information leakage (Theorem~\ref{thm:lowerbound}), and is supported by experiments (Figure~\ref{node}).

%% file: preliminary.tex
\section{Preliminaries}
We begin by introducing our notation. Let $\dist$ be a distribution over a sample space $\xxspace$. For an event $E\subseteq\xxspace$, we use $\dist(E)$ to denote the probability of $E$ under $\dist$. Given a feature transformation function $g:\xxspace\to\zzspace$ that maps instances from the input space $\xxspace$ to feature space $\zzspace$, we define $g_\sharp\dist\defeq \dist\circ g^{-1}$ to be the pushforward of $\dist$ under $g$, i.e., for any event $E'\subseteq\zzspace$, $g_\sharp\dist(E') \defeq \dist(\{x\in\xxspace\mid g(x)\in E'\})$. Throughout the paper we assume that the feature space is bounded, i.e., $\sup_{z\in\zzspace}\|z\|\leq R$. For two distributions $\dist$ and $\dist'$ over the same sample space $\Omega$, let $\dtv(\dist, \dist')$ be the total variation (TV) distance between them: $\dtv(\dist, \dist')\defeq \sup_{E\subseteq \Omega} |\dist(E) - \dist'(E)|$. The Wasserstein distance between $\dist$ and $\dist'$ is defined as $W_1(\dist,\dist') \defeq \sup_{\| f \|_L \leq 1} \left|\int_{\Omega} f d\dist - \int_{\Omega} f d\dist'\right|$, where $\|\cdot\|_L$ is the Lipschitz semi-norm of a real-valued function.

\textbf{Graph Neural Networks.} GNNs learn node and graph representations for predictions on nodes, relations, or graphs~\citep{scarselli2009graph}. The input is a graph $G=(V,E)$. Each node $u \in V$ has a feature vector $X_u$, and each edge $(u, v) \in E$ has a feature vector $X_{(u, v)}$.  GNNs iteratively compute a representation for each node.  Initially, the node representations are the node features: $X_u^{(0)} = X_u$.  In iteration $k=1,\ldots,K$, a GNN updates the node representations $X_u^{(k)}$ by aggregating the neighboring nodes' representations~\citep{gilmer2017neural}: $\forall k\in[K], u \in \mathcal{N}(v)$, 
\begin{equation*}
\small
    X_v^{(k)}  = \text{AGGREGATE}^{(k)} \left( \left\lbrace \left( X_u^{(k-1)}, X_v^{(k-1)}, X_{(u, v) } \right)\right\rbrace \right).
\end{equation*}
We can compute a graph representation $X_G$ by pooling the final node representations. Many GNNs, with different aggregation and graph pooling functions, have been proposed~\citep{defferrard2016convolutional, duvenaud2015convolutional, kipf2016semi, hamilton2017inductive, du2019graph, fey2019just,  Xu2020What, xu2021how, velivckovic2017graph,  cai2020graphnorm}.

\textbf{Threat Model.}
While being a powerful paradigm for learning node and graph representations for downstream tasks, GNNs also expose huge vulnerability to adversaries whose goal is to infer sensitive attributes from individual node representations. In light of this, throughout the paper we assume a black-box attack setting, where the adversary does not have direct access to the GNN encoder or knowledge of its parameters. Instead, it can obtain the node representations produced by the GNN encoder, with the goal to reconstruct a sensitive attribute $A_v$ of node $v$ by looking at its node representation $X_v^{(K)}$. Here, $A_v\in\{0, 1\}$\footnote{We assume binary  attributes for ease of exposition. Extension to categorical attributes is straightforward.} could be the user's age, gender, or income etc. The above setting is ubiquitous in \textit{server-client} paradigm where machine learning is provided as a service~\citep{ribeiro2015mlaas}. For example, in the distributed setting, when users' data is transmitted from local nodes to a centralized server, a malicious adversary could try to reveal users' sensitive attributes by eavesdropping the transmitted data. We emphasize that simply removing the sensitive attribute $A_v$ from the embedding $X_v$ is not sufficient, because $A_v$ may be inferred from some different but correlated features.

\textbf{Information Obfuscation.}
We study the problem above from an information-theoretic perspective. Let $\FF_A\defeq \{f: \RR^{d_K}\to\{0, 1\}\}$ denote the set of adversaries who have access to the node representations as input, and then output a guess of the sensitive attribute. Let $\dist$ be a joint distribution over the node input $X$, the sensitive attribute $A$, as well as the target variable $Y$. Let $\dist_a$ be the conditional distribution of $A = a$. We define the \emph{advantage}~\citep{goldwasser1996lecture} of the adversary as 
\begin{equation}
\small
\Adv_\dist(\FF_A)\defeq \sup_{f\in\FF_A} \left| \Pr_{\dist_1}(f(Z) = 1) - \Pr_{\dist_0}(f(Z) = 1)\right|,
\label{equ:adv}
\end{equation}
where $Z$ is the random variable that denotes \textit{node representations} after applying the GNN encoder to input $X$. Here, $f$ represents a particular adversary and the supresum in~\eqref{equ:adv} corresponds to the \textit{worst-case adversary} from a class $\FF_A$. If $\Adv_\dist(\FF_A) = 1$, then there exists an adversary who can almost surely guess the sensitive attribute $A$ by looking at the GNN representations $Z$. Hence, our goal is to design an algorithm that learns GNN representations $Z$ such that $\Adv_\dist(\FF_A)$ is small, which implies successful defense against adversaries. Throughout the paper, we assume that $\FF_A$ is symmetric, i.e., if $f\in\FF_A$, then $1 - f\in\FF_A$ as well.

%% file: algo.tex
\section{Information Obfuscation via Adversarial Training}
\label{sec:blgnn}
In this section, we first relate the aforementioned \textit{advantage} of the adversaries to a quantity that measures the \emph{ability of adversaries} in predicting a sensitive attribute $A$ by looking at the node representations $Z = X^{(K)}$. Inspired by this relationship, we proceed to introduce a minimax game between the GNN encoder and the worst-case adversary. We then extend this approach by considering adversaries that use score functions induced by the Wasserstein distance. To conclude this section, we analyze the trade-off between accuracy and information leakage, and provide theoretical guarantees for our algorithm.  

Given a symmetric set of adversaries $\FF_A$, we observe that $1 - \Adv_\dist(\FF_A)$ corresponds to the sum of best Type-I and Type-II \textit{inference error} any adversary from $\FF_A$ could attain:
\begin{equation*}
\small
1 - \Adv_\dist(\FF_A) =\inf_{f\in\FF_A}\left(\Pr_{\dist_0}(f(Z) = 1) + \Pr_{\dist_1}(f(Z) = 0)\right).
\end{equation*}
Hence, to minimize $\Adv_\dist(\FF_A)$, a natural strategy is to learn parameters of the GNN encoder so that it filters out the sensitive information in the node representation while still tries to preserve relevant information w.r.t.\ the target task of inferring $Y$. In more detail, let $Z = X^{(K)} = g(X)$ be the node representations given by a GNN encoder. We then propose the following unconstrained optimization problem with a trade-off parameter $\lambda > 0$ for information obfuscation:
\begin{equation}
    \begin{aligned}
    \min_{h, Z = g(X)}\max_{f\in\FF_A}\quad \eps_Y(h(Z)) - \lambda\cdot\eps_A(f(Z))
\end{aligned}
\label{equ:opt}
\end{equation}
Here we use $\eps_Y(\cdot)$ and $\eps_A(\cdot)$ to denote the cross-entropy error on predicting target task $Y$ and sensitive attribute $A$ respectively. In~\eqref{equ:opt}, $h$ is a parametrized network over the node representations $X^{(K)}$ for downstream task, and $f$ is the adversary. Note that the above formulation is different from the poisoning attacks on graphs literature~\citep{zugner2019certifiable}, where their goal is to learn a robust classifier under graph perturbation. It is also worth pointing out that the optimization formulation in~\eqref{equ:opt} admits an interesting game-theoretic interpretation, where two agents $f$ and $g$ play a game whose score is defined by the objective function in~\eqref{equ:opt}. Intuitively, $f$ could be understood as a simulated adversary seeking to minimize the sum of Type-I and Type-II inference error while the encoder $g$ plays against $f$ by learning features to remove information about the sensitive attribute $A$. Clearly, the hyperparameter $\lambda$ controls the trade-off between accuracy and information leakage. On one hand, if $\lambda\to 0$, we barely care about the defense of $A$ and devote all the focus to minimize the predictive error. On the other extreme, if $\lambda\to\infty$, we are only interested in defending against the potential attacks. 

\textbf{Wasserstein Variant}~~In practice, one notable drawback of the optimization formulation in~\eqref{equ:opt} is that the training of the adversary is unstable (cf. Fig.~\ref{fig:distance}), due to the fact that the TV distance is not continuous w.r.t.\ the model parameters of the encoder~\citep{arjovsky2017wasserstein}. Inspired by the Wasserstein GAN~\citep{arjovsky2017wasserstein}, we also propose a method for information obfuscation on GNNs with the Wasserstein distance, using its dual formulation. 
\begin{equation*}
\small
    \begin{aligned}
    \min_{h, Z = g(X)}\max_{\|f\|_L\leq 1}\quad \eps_Y(h(Z)) + \lambda \bigg|\int f(Z)~d\dist_0 - \int f(Z)~d\dist_1\bigg|
\end{aligned}
\label{equ:wopt}
\end{equation*}
The connection between~\eqref{equ:opt} and the above Wasserstein formulation lies in choice of the adversaries, as can be observed from the comparison between $\Adv_\dist(\FF_A)$ (advantage) and $W_1(\cdot, \cdot)$ (Wasserstein distance). In the latter case, the adversaries are constrained to be Lipschitz continuous. In both cases, we present the pseudocode in Fig.~\ref{pseudo} and Alg.~\ref{alg:gal}, which also includes extensions to \textit{neighborhood} and \textit{$n$-hop} attacks (cf.\ Section~\ref{section:3}).

\begin{figure}[!t]
\begin{lstlisting}[style = python]
z = encoder(graph)
z_rev = grad_reverse(z)
task_pred, adv_pred = decoder(z), attacker(z_rev)
loss_task = task_loss_f(task_pred, task_gt)
loss_adv = adv_loss_f(adv_pred, adv_gt)
if t % n == 0
    optimizer_adv.zero_grad()
    loss_adv.backward()
    optimizer_adv.step()
else:
    optimizer_task.zero_grad()
    loss_task.backward()
    optimizer_task.step()
\end{lstlisting}
\caption{Pseudocode for the core implementation of GAL (TV distance, node attack). Complete version is in Alg.~\ref{alg:gal}. }
\label{pseudo}
\vspace*{-1em}
\end{figure}

\subsection{Trade-off between Predictive Accuracy and Information Obfuscation}
\label{sec:trade-off}
The first term in the objective function of~\eqref{equ:opt} acts as an incentive to encourage GNN encoders to preserve task-related information. But will this incentive lead to the leakage of the sensitive attribute $A$? As an extreme case if the target variable $Y$ and the sensitive attribute $A$ are perfectly correlated, then it should be clear that there is a trade-off in achieving accuracy and preventing information leakage. Next, we provide an analysis to quantify this inherent trade-off. Theorem~\ref{thm:lowerbound} characterizes a trade-off between the cross-entropy error of task predictor and the advantage of the adversaries:
\begin{restatable}{theorem}{lowerbound}
Let $Z$ be the node representations produced by a GNN $g$ and $\FF_A$ be the set of all binary predictors. Define $\delta_{Y\mid A}\defeq |\Pr_{\dist_0}(Y=1) - \Pr_{\dist_1}(Y = 1)|$. Then for a classifier $h$ such that $\|h\|_L\leq C$, 
\begin{align}
    & \eps_{Y\mid A = 0}(h\circ g) + \eps_{Y\mid A = 1}(h\circ g) \nonumber\\
\geq     & ~\delta_{Y\mid A} - C\cdot W_1(g_\sharp\dist_0, g_\sharp\dist_1) \nonumber\\
\geq    &  ~\delta_{Y\mid A} - 2RC\cdot \Adv_\dist(\FF_A).
\end{align}
\label{thm:lowerbound}
\vspace*{-2em}
\end{restatable}
\paragraph{Remark.}
Recall that $R$ is a bound of the radius of the feature space, i.e., $\sup_{z\in\zzspace}\|z\|\leq R$. First, $\eps_{Y\mid A = a}(h\circ g)$ denotes the conditional cross-entropy error of predicting $Y$ given $A = a\in\{0, 1\}$. Hence the above theorem says that, up to a certain threshold given by $\delta_{Y\mid A}$ (which is a task-specific constant), any target predictor based on the features given by GNN $g$ has to incur a large error on at least one of the sensitive groups. Specifically, the smaller the adversarial advantage $\Adv_\dist(\FF_A)$ or the Wasserstein distance $W_1(g_\sharp\dist_0, g_\sharp\dist_1)$, the larger the error lower bound. The lower bound in Theorem~\ref{thm:lowerbound} is \textit{algorithm-independent}, and considers the strongest possible adversary, hence it reflects an \emph{inherent} trade-off between task utility and information obfuscation. Moreover,  Theorem~\ref{thm:lowerbound} does not depend on the marginal distribution of the sensitive attribute $A$. Instead, all the terms in our result only depend on the conditional distributions given $A = 0$ and $A = 1$. As a simple corollary, the overall error also admits a lower bound:
\begin{restatable}{corollary}{glowerbound}
Assume the conditions in Theorem~\ref{thm:lowerbound} hold. Let $\alpha\defeq \Pr_\dist(A = 0)$, then
\begin{align*}
\small
\eps_{Y}(h\circ g) &\geq \min\{\alpha, 1-\alpha\}\big(\delta_{Y\mid A} - C\cdot W_1(g_\sharp\dist_0, g_\sharp\dist_1)\big) \\
& \geq \min\{\alpha, 1-\alpha\}\big(\delta_{Y\mid A} - 2RC\cdot \Adv_\dist(\FF_A)\big).
\end{align*}
\label{cor:lowerbound}
\vspace*{-2em}
\end{restatable}
Our lower bounds in Theorem~\ref{thm:lowerbound} and Corollary~\ref{cor:lowerbound} capture the price we have to pay for obfuscation.

\subsection{Guarantees Against Information Leakage}
Next, we provide guarantees for information obfuscation  using~\eqref{equ:opt}. The analysis on the optimization trajectory of a general non-convex-concave game~\eqref{equ:opt} is still an active area of research~\citep{daskalakis2018limit,nouiehed2019solving} and hence beyond the scope of this paper. Therefore, we assume that we have access to the minimax stationary point solution of~\eqref{equ:opt}, and focus on understanding how the solution of~\eqref{equ:opt} affects the effectiveness of our defense. 

In what follows we analyze the true error that a worst-case adversary has to incur in the limit, when both the task classifier and the adversary have unlimited capacity, i.e., they can be any randomized functions from $\zzspace$ to $\{0, 1\}$. To this end, we also use the population loss rather than the empirical loss in our objective function. Under such assumptions, given any node embedding $Z$ from a GNN $g$, the worst-case adversary is the \emph{conditional mean}: $\min_{f\in\FF_A}\eps_A(f\circ g) = H(A\mid Z),~\argmin_{f\in\FF_A}\eps_A(f\circ g) = \Pr(A = 1\mid Z)$.
It follows from a symmetric argument that $\min_{h\in\HH}\eps_Y(h\circ g) = H(Y\mid Z)$. Hence we can  simplify the objective \eqref{equ:opt} to the following form where the only variable is the embedding $Z$:
\begin{equation}
    \underset{Z = g(X)}{\min}\quad H(Y\mid Z) - \lambda\cdot H(A\mid Z)
    \label{equ:feature}
\end{equation}
We can now analyze the error that has to be incurred by the worst-case adversary:
\begin{restatable}{theorem}{guarantee}
    Let $Z^*$ be the optimal GNN node embedding of \eqref{equ:feature}. Define $\alpha\defeq \Pr_\dist(A = 0)$, $H^*\defeq H(A\mid Z^*)$ and $W_1^* \defeq W_1(Z^*\mid A = 0, Z^*\mid A = 1)$. Then 1). For any adversary $f:\zzspace\to\{0, 1\}$, $\Pr(f(Z)\neq A) \geq H^*/ 2\lg(6/H^*)$, 2). For any Lipschitz adversary $f$ such that $\|f\|_L\leq C$, $\Pr(f(Z)\neq A) \geq \min\{\alpha, 1-\alpha\}(1 - CW_1^*)$.
    \label{thm:guarantee}
\end{restatable}
Theorem~\ref{thm:guarantee} shows that whenever the conditional entropy $H^* = H(A\mid Z^*)$ is large or the Wasserstein distance $W_1^*$ is small, the inference error incurred by any (randomized) adversary has to be large. It is worth pointing out that when $W_1^* = 0$, the second inference error lower bound reduces to $\min\{\alpha, 1-\alpha\}$, which is attained by an adversary that uses constant prediction of the sensitive attribute $A$, i.e., this adversary always guesses the majority value of $A$. Hence, Theorem~\ref{thm:guarantee} justifies the use of GAL for information obfuscation. As a final note, recall that the representations $Z$ appearing in the bounds above depend on the \emph{graph structure}, and the inference error in Theorem~\ref{thm:guarantee} is over the representations $Z$. Together, this suggests that the defense could potentially be applied against \textit{neighborhood attacks}, which we provide in-depth empirical validation in Section~\ref{section:3}.

%% file: experiment.tex
\section{Experiments}

\begin{algorithm}[t!]
    \caption{Full description of GAL, \color{orange} TV \color{black} distance and \color{blue} Wasserstein \color{black} distance.  {The node pairing policy ${Q}$ decides how the predicted node-level sensitive attributes $A_v$ matches against the targets $A_w$ for all nodes in the graph: in node-level obfuscation, it is simply the identity transformation; for neighborhood and N-Hop attack, pairing is done with the proposed Monte-Carlo probabilistic algorithm on $G$. Note that for the expository analysis above , ${L}_{\text{adversary}} $ is equivalent to $ \eps_A$. }
    }
    \label{alg:gal}
    \begin{algorithmic}
    \STATE  {\bfseries Input: }{$G = (V, E)$: input graph; $Y$: node-level target; $A$: node-level sensitive attributes; $g$: GNN encoder; $h$: task decoder; $f$: worst-case adversary; ${L}_{\text{task}}$: task loss; ${L}_{\text{adversary}}$: adversary loss; ${Q}$: node-pairing policy }
    \STATE  {\bfseries Input: }{$n$: training iterations; $\alpha$: learning rate; $m$: batch size; $n_a$: adversary per task; $\lambda$: trade-off parameter}
    \STATE  {\bfseries Input: }{$\theta_g, \theta_h, \theta_f$: corresponding parameters for network}
    \STATE  {\bfseries Input: \color{blue} (Required by Wasserstein) } \color{blue} {$c$: clipping parameter; $C_A$: classes in $A$} \color{black}
      \STATE $t = 0$ 
       \REPEAT
      \STATE $z  = g(V, E)$ \COMMENT{via AGGREGATE}
      \STATE $z_{\text{rev}}  = \text{GradReverse}(z)$\;\;\quad\quad\citep{ganin2016domain}
      \STATE $Y'  = h(z)$
      \STATE $A'  = f(z_{\text{rev}})$
     
     \IF{$t \bmod n_a = 0$} 
      \STATE ${B} = (A_p, A'_p) = {Q}(A, A')$
      \STATE  \color{orange} (TV): $l$ $ = {L}_{\text{adversary}} (A_p, A'_p)$ \color{black}, 
      \STATE \color{blue}  (Wasserstein): $\overline{p} = \text{mean}(\{ (A_{p}^{'(a)}) \odot {1}_{A_p = a} \}:{a \in C_A} ) $
      \STATE $l$ $ = | \max_{a \in C_A} (\overline{p}_a) - \min_{a \in C_A} (\overline{p}_a) |$ \color{black} 
      \STATE $g_{\text{g}}, g_{\text{f}}  = \nabla_g l, \nabla_f l$ 
      \STATE $\theta_{\text{g}}, \theta_{\text{f}}  = \theta_{\text{g}} + \lambda \cdot \text{SGD}(\theta_g, g_g) , \theta_{\text{f}} + \lambda \cdot \text{SGD}(\theta_f, g_f) $
      \STATE \color{blue} (Wasserstein): $\theta_f  = \text{clip}(\theta_f, -c, c)$ \color{black} 
     \ELSE
       \STATE $g_{\text{g}}, g_{\text{h}}  = \nabla_g {L}_{\text{task}} (Y, Y'), \nabla_h {L}_{\text{task}} (Y, Y')$
       \STATE $\theta_{\text{g}}, \theta_{\text{h}}  = \theta_{\text{g}} + \alpha \cdot \text{SGD}(\theta_g, g_g) , \theta_{\text{h}} + \alpha \cdot \text{SGD}(\theta_h, g_h)$
     \ENDIF
      \STATE $t  = t + 1$ 
     \UNTIL {$t >= n$}
    \end{algorithmic}
\end{algorithm}
  
\begin{table*}[tb]
\vspace{-0.1in}
  \centering
  \caption{\small \textbf{Performance of prediction and obfuscation on Movielens-1M}. We first perform link prediction as the main task, then probe with adversaries performing node classification tasks. The statistics are reported as \{adversary performance, F1/AUC\}/\{task performance, RMSE\}. Best-performing methods and best obfuscation results are highlighted in bold in the first table. Since \citet{bose2019compositional}'s method does not allow tradeoff tuning, incompatible fields are left unfilled.}
  \label{task2}
  \resizebox{0.9\linewidth}{!}{%
    \begin{tabular}{ l | c c c c c }
    \toprule
    Encoder & $\lambda$ & Gender-F1/Task-RMSE & Gender-AUC/Task-RMSE & Age-F1/Task-RMSE & Occupation-F1/Task-RMSE \\ [0.5ex] 
    \midrule
    ChebNet-TV         & 0.5 & 0.617 / 0.881 & 0.642 / 0.881 & 0.141 / 0.875 & 0.044 / 0.868 \\
    ChebNet-W          & 0.5 & \textbf{0.593} / 0.872 & \textbf{0.605} / 0.872 & \textbf{0.137} / 0.888 & \textbf{0.033} / 0.908 \\
    GraphSAGE-TV       & 0.5 & 0.679 / 0.866 & 0.680 / 0.866 & 0.236 / 0.869 & 0.055 / 0.871 \\
    GraphSAGE-W        & 0.5 & 0.691 / 0.893 & 0.715 / 0.893 & 0.226 / 0.901 & 0.050 / 0.916 \\
    \midrule
    Bose and Hamilton  &   & 0.667 / 0.874 & 0.678 / 0.874 & 0.188 / 0.874 & 0.051 / 0.874 \\\bottomrule
    \end{tabular}}

\vspace{0.5em}
    
  \begin{minipage}{.5\linewidth}
      \centering
      \resizebox{.95\linewidth}{!}{
        \begin{tabular}{ l | c c c c c }
        \toprule
        ChebNet-TV         & 0 & 0.692 / 0.852 & 0.707 / 0.852 & 0.288 / 0.852 & 0.078 / 0.851 \\
        ChebNet-W          & 0 & 0.693 / 0.852 & 0.707 / 0.852 & 0.286 / 0.852 & 0.077 / 0.851 \\
        GraphSAGE-TV       & 0 & 0.728 / 0.849 & 0.735 / 0.849 & 0.293 / 0.849 & 0.080 / 0.851 \\
        GraphSAGE-W        & 0 & 0.724 / 0.849 & 0.734 / 0.849 & 0.293 / 0.849 & 0.081 / 0.851 \\
        \bottomrule
        \end{tabular}}
    \end{minipage}%
    \begin{minipage}{.5\linewidth}
      \centering
        \resizebox{.95\linewidth}{!}{
        \begin{tabular}{ l | c c c c c }
        \toprule
        ChebNet-TV         & 4 & 0.505 / 1.280 & 0.532 / 1.280 & 0.123 / 1.301 & 0.016 / 1.241 \\
        ChebNet-W          & 4 & 0.485 / 1.258 & 0.526 / 1.258 & 0.104 / 1.296 & 0.025 / 1.353 \\
        GraphSAGE-TV       & 4 & 0.675 / 0.900 & 0.683 / 0.900 & 0.200 / 0.904 & 0.050 / 0.898 \\
        GraphSAGE-W        & 4 & 0.471 / 0.970 & 0.516 / 0.970 & 0.074 / 1.080 & 0.010 / 1.117 \\
        \bottomrule
        \end{tabular}}
\end{minipage} 

  \end{table*}

In this section, we demonstrate the effectiveness of GAL for information obfuscation on graphs. Specifically, we address the following three questions: 
\begin{itemize}[leftmargin=2.5em,noitemsep,nolistsep]
  \item[\ding{212}] \ref{section:1}: Is GAL effective across different tasks, distances and GNN encoder architectures? How do the TV and Wasserstein variants compare during training?
  \item[\ding{212}] \ref{section:2}: What is the landscape of the tradeoff with respect to the hyperparameter $\lambda$?
  \item[\ding{212}] \ref{section:3}: Can GAL also defend neighborhood and $n$-hop attacks?
\end{itemize}

We empirically confirm all three questions. To stress test the robustness of GAL, we consider a variety of tasks and GNN encoder architectures.  Specifically, we experiment on 5 link-prediction benchmarks (Movielens-1M, FB15k-237, WN18RR, CiteSeer, Pubmed) and 1 graph regression benchmark (QM9), which covers both obfuscation of \emph{single} and \emph{multiple} attributes. More importantly, our goal
is not to challenge state-of-the-art training schemes, but to observe the effect in reducing adversaries' accuracies while maintaining good performance of the downstream task. Our experiments can be classified into three major categories.

\textbf{Robustness.}~~We run GAL on all six datasets under a variety of GNN architectures, random seeds, distances and trade-off parameter $\lambda$. We report average performance over five runs. We perform ablation study on widely used GNN architectures \citep{velivckovic2017graph, kipf2016semi, compgcn, defferrard2016convolutional, gilmer2017neural}, and select the best-performing GNNs in the task: CompGCN is specifically designed for knowledge-graph-related applications.

\begin{table}[tb]
\vspace*{-1.0em}
\small
  \centering
  \caption{\small \textbf{Summary of benchmark dataset statistics}, including number of nodes $|V|$, number of nodes with sensitive attributes  
  $|S|$, number of edges $|E|$, node-level non-sensitive features $d$, target task and adversarial task, and whether the experiment concerns with
  multi-set obfuscation.}
  \label{summary}
  \resizebox{0.5\textwidth}{!}{%
  \begin{tabular}{ l l l l l l l | l}
    \toprule
    \textsc{Dataset}      & $|V|$       & $|S|$       & $|E|$       & $d$ &  Multi? & \textsc{Metrics} & \textsc{Adversary} \\ [0.5ex] 
    \midrule
    \textsc{CiteSeer}     & $3,327$     & $3,327$     & $4,552$     & $3,703$  &\xmark& \text{AUC}           & Macro-F1  \\
    \textsc{Pubmed}       & $19,717$    & $19,717$    & $44,324$    & $500$    &\xmark& \text{AUC}           & Macro-F1  \\
    \textsc{QM9}          & $2,383,055$ & $2,383,055$ & $2,461,144$ & $13$     &\xmark& \text{MAE}           & MAE   \\
    \textsc{ML-1M}        & $9,940$     & $6,040$     & $1,000,209$ & $1$ (id) &\xmark& \text{RMSE}          & Macro-F1/AUC  \\
    \textsc{FB15k-237}    & $14,940$    & $14,940$    & $168,618$   & $1$ (id) &\cmark& \text{MRR}   & Macro-F1  \\
    \textsc{WN18RR}       & $40,943$    & $40,943$    & $173,670$   & $1$ (id) &\xmark& \text{MRR}   & Macro-F1  \\
    \bottomrule
   \end{tabular}}
   \vspace*{-2.0em}
\end{table}

\textbf{Trade-off.}~~We compare performance under a wide range of $\lambda$ on Movielens-1M. We show that GAL defends the inference attacks to a great extent while only suffering minor losses to downstream task performance. We also empirically compare the training trajectories of the TV and Wasserstein variants, and confirm that the Wasserstein variant often leads to more stable training.

\textbf{Neighborhood Attacks.}~~In this setting, an adversary also has access to the embeddings of \textit{neighbors} of a  node, e.g. the adversary can infer 
sensitive attribute $A_v$ from $X_w^{(K)}$ (instead of $X_v^{(K)}$) such that there is a path between $v$ and $w$. Since GNN's neighborhood-aggregation paradigm may introduce such information leakage, adversaries shall achieve nontrivial performance. We further generalize this to an $n$-hop scenario. %

In all experiments, the adversaries only have access to the training set labels along with embeddings from the GNN, and the performance is measured on the held-out test set. A summary of the datasets, including graph attributes, task, and adversary metrics, is in Table \ref{summary}. Detailed experimental setups may be found in Appendix~\ref{sec:setup}. Overall, our results have successfully addressed all three questions, demonstrating that our framework is attractive for node- and neighborhood-level attribute obfuscation across downstream tasks, GNN architectures, and trade-off parameter $\lambda$.

\begin{figure}[!t]
\vspace*{-1.0em}
  \centering
  \includegraphics[width=\linewidth]{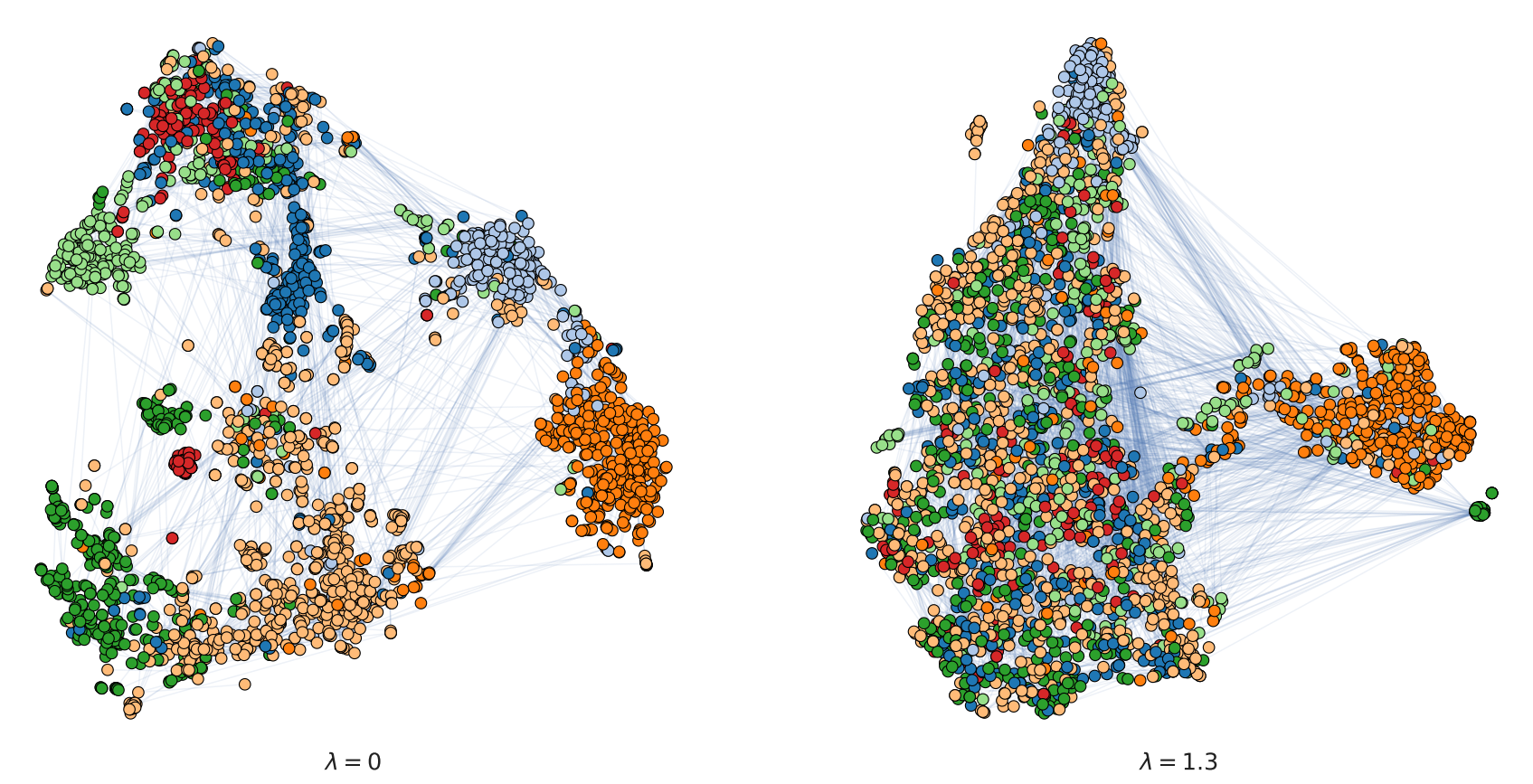}
  \caption{\small \textbf{GAL effectively protects sensitive information}. Both panels show t-SNE of feature representations under different trade-off parameters $\lambda$ (0 vs. $1.3$). 
  Node colors represent node classes.}
  \label{tsne}
\end{figure}
\begin{figure}[htb] 
\vspace*{-1.5em}
\centering
\subfloat[Adversary: Macro F-1]{%
    \includegraphics[width=0.5\linewidth]{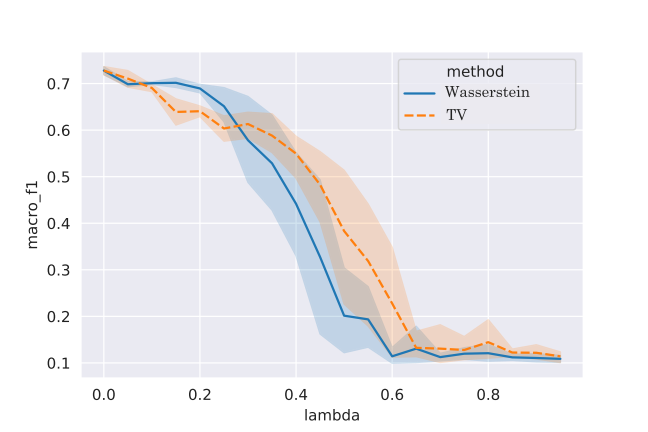}%
    \label{fig:a}%
    }%
\hfill%
\subfloat[Task: Binary cross-entropy]{%
    \includegraphics[width=0.5\linewidth]{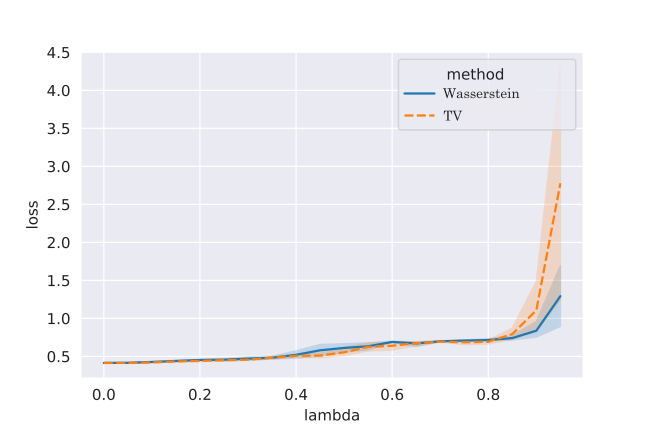}%
    \label{fig:b}%
    }%
\caption{\textbf{Performance of Wasserstein vs. TV for node-level information obfuscation} on Cora dataset under different $\lambda$, with $95\%$ confidence interval over five runs.  Wasserstein offers a more stable task performance while producing a better obfuscation compared to TV.}
\label{fig:distance}
\vspace*{-1em}
\end{figure}

\begin{figure*}[t!]
  \centering
  \includegraphics[width=0.85\linewidth]{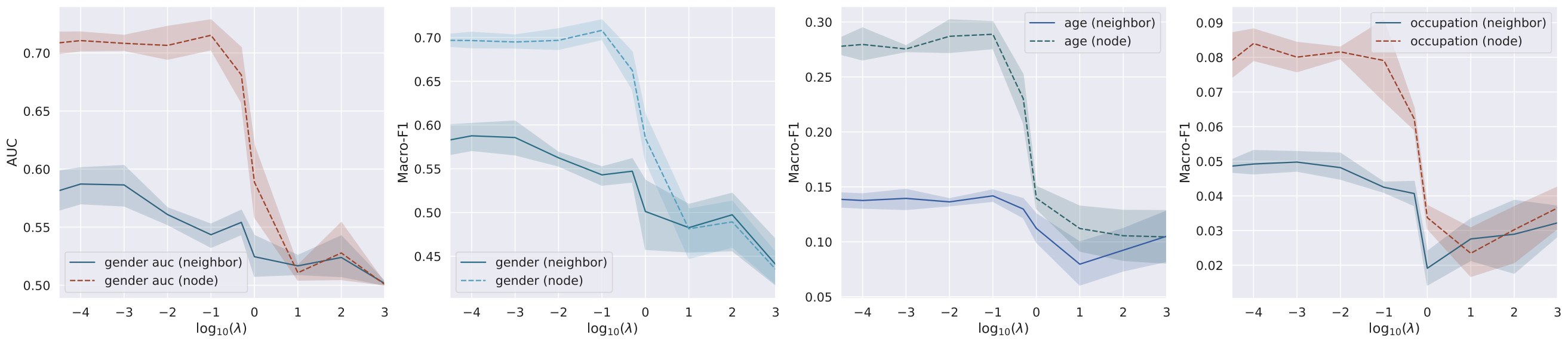}
  \includegraphics[width=0.85\linewidth]{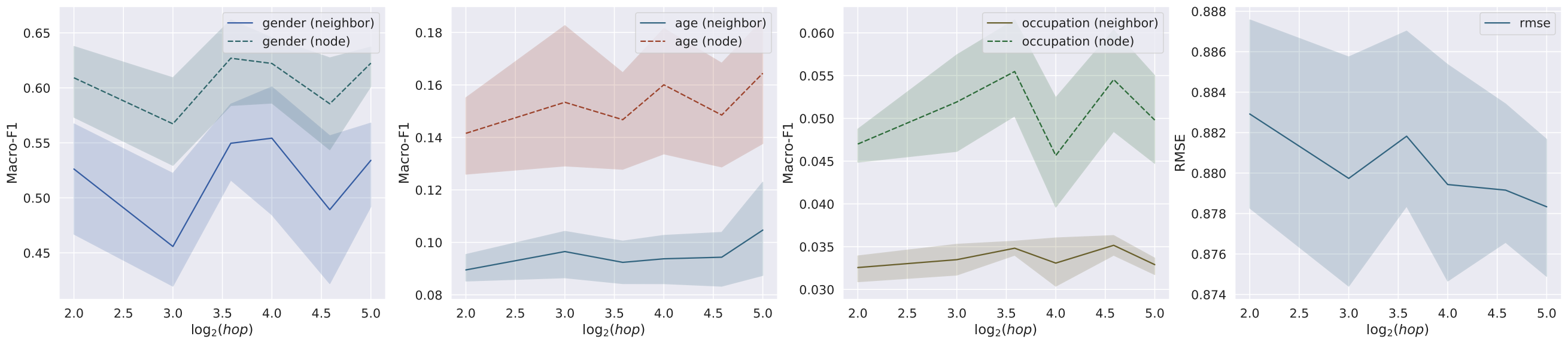}
  \caption{\small \textbf{Performance of neighborhood-level (top) and $n$-hop (bottom) attribute obfuscation} on a 3-layer/2-layer ChebNet with Movielens-1M dataset under different $\lambda$/``hop`` respectively. Band represents 95\% confidence interval over five runs. For $n$-hop experiments, $\lambda$ is fixed to be 0.8, and distance is Total Variation. }
  \label{neighbor}
  \vspace*{-1em}
\end{figure*}

\begin{figure}[t!] 
  \centering
    \includegraphics[width=\linewidth]{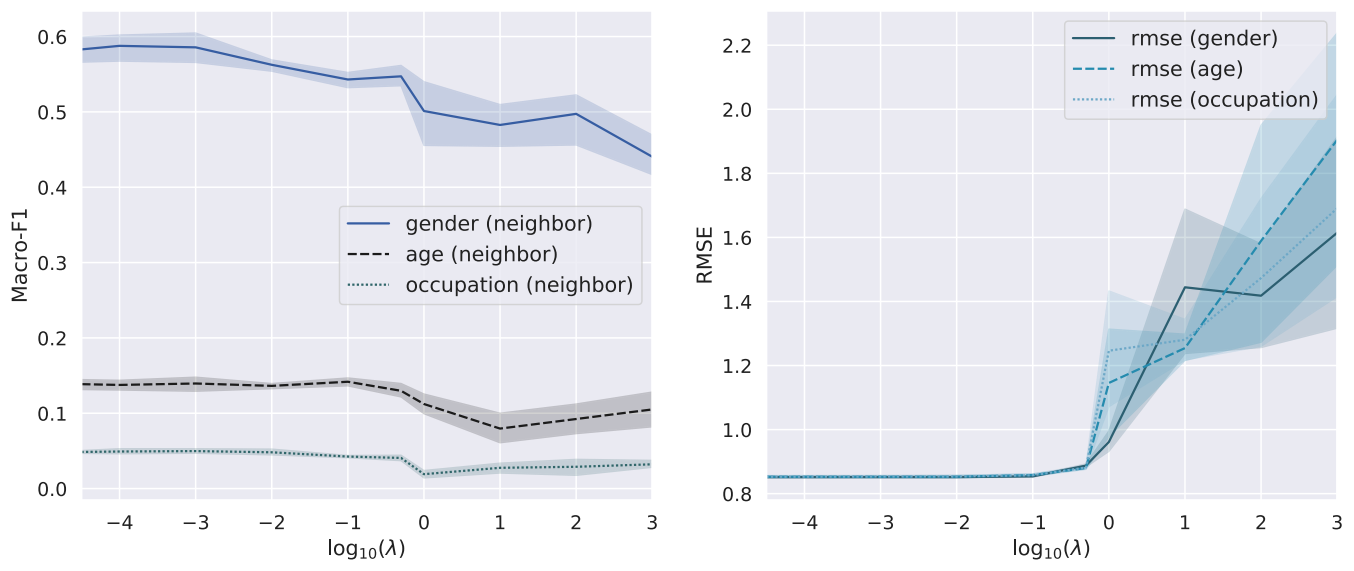}
    \caption{\textbf{Performance of node-level attribute obfuscation} on Movielens-1M dataset under different $\lambda$, with $95\%$ confidence interval over five runs.}
    \label{node}
\vspace*{-1.5em}
\end{figure} 

\subsection{Robust Node-level Information Obfuscation}
\label{section:1}
We first demonstrate that our framework is effective across a wide variety of GNN architectures, tasks and training methods. Quantitatively, we report task performances $T_{i}$ as well as adversary performances $A_{i}$, $\forall i\in[3]$ both specific to respective datasets in Table~\ref{task2} and Table~\ref{task}, where we use the subscript $i$ to denote the $i$-th experimental setting. Note that fixed-embedding \citep{compfair} results are only reported for Movielens-1M dataset because on other datasets, the experimental setups are not suitable for comparison. Across all datasets, we witness a significant drop in adversary's performance with a rather minor decrease in the downstream task performance.

Qualitatively, the effect of adversarial training is also apparent: in Figure \ref{tsne}, we visualized the t-SNE~\citep{tsne}-transformed representations of graph embeddings under different regularization strengths on the Cora dataset. We observe that under a high regularization strength, node-level attributes are better mixed in the sense that nodes belonging to different classes are harder to separate from each other on a graph-structural level. 

 \begin{table*}[]
\vspace{-0.1in}
  \small
  \centering
   \caption{\small \textbf{Performance of predictions and obfuscation on benchmark datasets}. $\lambda_i, T_i, A_i$ represents the tradeoff parameter, task performance and adversary performance in the $i$-th experiment, respectively. Best-performing GNNs and best obfuscation results are highlighted in bold. The ``---`` mark stands for ``same as above.`` }
  \label{task}
  \resizebox{\linewidth}{!}{%
  \begin{tabular}{ l l l l | c c c c c c  c c c }
    \toprule
             & GNN Encoder   & Method & Adversary & $\lambda_1$ & $\lambda_2$ & $\lambda_3$  & $T_1$ & $T_2$ &  $T_3$ & $A_1$ &  $A_2$ &  $A_3$ \\ [0.5ex] 
    \midrule 
    \textsc{Pubmed}  
   & \text{GCN}    & Total Variation & Doc. Class & 0.5 & 0.85  & 0.95 & $\mathbf{0.959}$  & $0.930$  & $0.818$  & $0.800$  & $0.798$  & $0.762$ \\
   &               & Wasserstein    & Doc. Class & ---  & ---  & ---  & $0.959$  & $0.895$  & $0.733$  & $0.798$  & $0.780$  & $0.596$ \\
   & \text{GAT}    & Total Variation & Doc. Class & ---  & ---  & ---  & $0.924$  & $0.873$  & $0.810$  & $0.782$  & $0.794$  & $0.730$ \\
   &               & Wasserstein    & Doc. Class & ---  & ---  & ---  & $0.928$  & $0.846$  & $0.785$  & $0.778$  & $0.766$  & $0.741$ \\
   & \text{ChebNet}       
                   & Total Variation & Doc. Class & ---  & ---  & ---  & $0.933$  & $0.872$  & $0.757$  & $0.798$  & $0.788$  & $0.747$ \\
   &               & Wasserstein    & Doc. Class & ---  & ---  & ---  & $0.919$  & $0.727$  & $0.739$  & $0.780$  & $0.641$  & $\mathbf{0.568}$ \\
    \bottomrule
    \midrule
    \textsc{CiteSeer}  
   & \text{GCN}            & Total Variation & Doc. Class & 0.75 &  1   & 1.5  & $0.867$  & $0.694$  & $0.500$  & $0.785$  & $0.731$  & $0.652$ \\
   &                       & Wasserstein    & Doc. Class & ---  & ---  & ---  & $0.814$  & $0.623$  & $0.504$  & $0.644$  & $0.608$  & $0.407$ \\
   & \text{GAT}            & Total Variation & Doc. Class & ---  & ---  & ---  & $\mathbf{0.896}$  & $0.553$  & $0.505$  & $0.789$  & $0.273$  & $0.168$ \\
   &                       & Wasserstein    & Doc. Class & ---  & ---  & ---  & $0.811$  & $0.519$  & $0.501$  & $0.656$  & $0.278$  & $\mathbf{0.125}$ \\
   & \text{ChebNet}       & Total Variation & Doc. Class & ---  & ---  & ---  & $0.776$  & $0.678$  & $0.500$  & $0.735$  & $0.725$  & $0.642$ \\
   &                       & Wasserstein    & Doc. Class & ---  & ---  & ---  & $0.732$  & $0.614$  & $0.500$  & $0.661$  & $0.568$  & $0.403$ \\
    \bottomrule
    \midrule
    \textsc{QM9}     & \text{MPNN}           & Total Variation & Polarizability    & 0    & 0.05 & 0.5 & $\mathbf{0.121}$          & $0.132$          & $0.641$          & $1.054$          & $1.359$          & $\mathbf{3.100}$ \\
    \bottomrule
    \midrule
    \textsc{WN18RR}  & \text{CompGCN}         & Total Variation & Word Sense        & 0    & 1.0  & 1.5 & $\mathbf{0.462}$   & $0.437$    & $0.403$    & $0.208$     & $\mathbf{0.131}$     & $0.187$ \\
                     &                        & Total Variation & POS tag           & ---  & ---  & --- & ---   & ${0.430}$        & $0.395$    & $0.822$     & $\mathbf{0.607}$     & $0.705$ \\
    \bottomrule
    \midrule
    \textsc{FB15k-237}  & \text{CompGCN}       & Total Variation & Ent. Attr.   & 0    & 1.0  & 1.5 & $\mathbf{0.351}$  & $0.320$          & $0.319$          & $0.682$   & $0.641$   & $\mathbf{0.630}$ \\
    \bottomrule
    \end{tabular}}
  \end{table*}

\subsection{Fine-grained Trade-off Tuning with $\lambda$}
\label{section:2}
\vspace*{-0.5em}
We show that the hyper-parameter $\lambda$ is indeed a good trade-off hyper-parameter that well balances downstream task performance and obfuscation effectiveness.
We perform experiments with ChebNet and the Movielens-1M dataset, and test $\lambda$ values ranging from $0$ to $10^3$. Our results in Figure \ref{node} show that
as $\lambda$ increases, it is indeed more difficult for the adversary to extract sensitive information from GNN embeddings; the downstream task performance decreases steadily but with a small rate. When
$\lambda$ is larger than $10$, training destabilizes further due to inherent difficulties in optimization, which is demonstrated by the higher variance in task performance.
However, this trade-off is less visible for larger graphs like FB15k-237 and WN18RR, where the effects of lambda tends to be less
monotonic. The detailed results are in Table \ref{task} and Figure \ref{fig:distance}.

\subsection{Neighborhood Attack and $n$-Hop Genrealization}
\label{section:3}

Finally, we consider the scenario where the adversary seeks to extract node-level information from embeddings of \textit{neighbors}.
 This is, in particular, a serious concern in GNNs where the message passing layers incorporate neighborhood-level information and hence the information leakage. Our results in Figure \ref{neighbor} verify that, without any form of obfuscation ($\lambda=0$), the adversary indeed achieves nontrivial performance by accessing the embeddings of neighbors on the Movielens-1M dataset with a simple 2-layer ChebNet. We further run comprehensive experiments where the obfuscation targets neighborhood-level embeddings, and we observe the degradation of the adversary's performance (both on the node- and neighbor-level embeddings) as $\lambda$ increases. 
The results demonstrate that under low values of $\lambda$, the neighborhood-level obfuscation has more effects on neighborhood adversaries than node adversaries, 
suggesting that node-level embeddings are less protected. However, as $\lambda$ continues to increase, the degradation in the performance of node adversaries is more visible. An extension of this setup is to consider neighborhood-level attacks as ``single-hop'', where we can naturally define ``$n$-hop'' as adversaries having accesses to only embeddings of nodes that are $n$-distant away from the target node.

Since finding $n$-distant neighbors for arbitrary nodes in a large graph on-the-fly is computationally inefficient during training (in the sense that the complexity bound involves
$|E|$ and $|V|$), we propose a Monte-Carlo algorithm that probabilistically finds such neighbor in $O(n^2)$ time, the details of which can be found in appendix. 
We report results under different ``hops`` with the same encoder-$\lambda$ setting on Movielens-1M, shown in Figure \ref{neighbor}. 
In general, we observe that as ``hop`` increases, the retrieved embedding contains less information about the target node. 
Therefore, adversarial training's effect will have less effect on the degradation of the target task, which is demonstrated by the steady decrease of RMSE. 
The fluctuations in the trend of node- and neighborhood-level adversaries are due to the probabilistic nature of the Monte-Carlo algorithm used to sample neighbors, 
where it may end up finding a much closer neighbor than intended, destabilizing the training process.
This is another trade-off due to limited training time, yet the general trend is still visible, certifying our assumptions.

%% file: conclusion.tex
\section{Other Related Work}
\textbf{Adversarial Attack on Graphs.} Our problem formulation, i.e., information obfuscation, is \textit{significantly different} from  adversarial attacks on graphs~\citep{Bojchevski2019AdversarialAO,Ma2019AttackingGC,Xu2019TopologyAA,Dai2018AdversarialAO,Chang2019TheGB}. Works on adversarial attacks focus on perturbations of a graph, e.g., by adding or removing edges, that maximize the loss of ``victim'' nodes for GNNs. In contrast, adversaries in our framework do not alter graph structure; instead, they seek to extract critical information from the GNN node embeddings.

\textbf{Differential Privacy.} A related but \textit{different} notion of privacy is differential privacy~\citep{dwork2014algorithmic}, which aims at defending the so-called membership inference attack~\citep{shokri2017membership,nasr2018machine}. Their goal is to design randomized algorithms so that an adversary cannot guess whether an individual appears in the training data by looking at the output of the algorithm. We instead tackle a information obfuscation problem where the adversary's goal is to infer some attribute of a node in a graph, and our goal is to find embedding of the data, so that the adversary cannot accurately infer a sensitive attribute from the embedding.

\textbf{Information Bottleneck.}
Our work is also related to but \textit{different from} the information bottleneck method, which seeks to simultaneously compress data while preserving information for target tasks~\citep{tishby2000information,alemi2016deep,tishby2015deep}.  The information bottleneck framework could be understood as maximizing the following objective: $I(Y; Z) - \beta I(X; Z)$. Specifically, there is no sensitive attribute $A$  in the formulation. In contrast, the minimax framework in this work is about a trade-off problem, and the original input $X$ does not appear in our formulation: $I(Y; Z) - \beta I(A; Z)$. Overall, despite their surface similarity, our adversarial method is significantly different from that of the information bottleneck method in terms of both formulation and goals. 

\section{Conclusion}
In this work, we formulate and address the problem of information obfuscation on graphs with GNNs.
Our framework, termed GAL, introduces a minimax game between the desired GNN encoder and the worst-case adversary.
GAL creates a strong defense against information leakage in terms of a provable lower bound, while only suffering a marginal loss in task performance. We also show an information-theoretic bound for the inherent trade-off between accuracy and obfuscation. Experiments show that GAL  perfectly complements existing algorithms deployed in downstream tasks to address information security concerns on graphs.

\section*{Acknowledgements}
We thank the anonymous reviewers for their insightful feedback and suggestions. RS was partly supported by the NSF IIS1763562 and ONR Grant N000141812861. KX and SJ were supported by NSF CAREER award 1553284 and NSF III 1900933. HZ thanks the DARPA XAI project, contract \#FA87501720152, for support. GG thanks Microsoft Research for support.

%% file: appendix.tex
\onecolumn 

\section{Proofs}
In this section we provide the detailed proofs of both theorems in the main text. We first rigorously show the relationship between the adversarial advantage and the inference error made by a worst-case adversarial:
\begin{claim}
  $1 - \Adv_\dist(\FF_A) =\inf_{f\in\FF_A}\left(\Pr_\dist(f(Z) = 1\mid A = 0) + \Pr_\dist(f(Z) = 0\mid A = 1)\right)$.
\end{claim}
\begin{proof}
Recall that $\FF_A$ is symmetric, hence $\forall f\in\FF_A$, $1 - f\in\FF_A$ as well:
\begin{align*}
  1 - \Adv_\dist(\FF_A) &= 1 - \sup_{f\in\FF_A}\left|\Pr_\dist(f(Z) = 0\mid A = 0) - \Pr_\dist(f(Z) = 0\mid A = 1)\right| \\
                        &= 1 - \sup_{f\in\FF_A}\left(\Pr_\dist(f(Z) = 0\mid A = 0) - \Pr_\dist(f(Z) = 0\mid A = 1)\right) \\
                        &= \inf_{f\in\FF_A}\left(\Pr_\dist(f(Z) = 1\mid A = 0) + \Pr_\dist(f(Z) = 0\mid A = 1)\right),
\end{align*}
where the second equality above is because 
\begin{equation*}
  \sup_{f\in\FF_A}\left(\Pr_\dist(f(Z) = 0\mid A = 0) - \Pr_\dist(f(Z) = 0\mid A = 1)\right)  
\end{equation*}
is always non-negative due to the symmetric assumption of $\FF_A$.
\end{proof}

Before we prove the lower bound in Theorem~\ref{thm:lowerbound}, we first need to introduce the following lemma, which is known as the data-processing inequality of the TV distance. 
\begin{lemma}[Data-processing of the TV distance]
\label{lemma:dpi}
Let $\dist$ and $\dist'$ be two distributions over the same sample space and $g$ be a Markov kernel of the same space, then $\dtv(g_\sharp\dist, g_\sharp\dist')\leq \dtv(\dist, \dist')$, where $g_\sharp\dist (g_\sharp\dist')$  is the pushforward of $\dist (\dist')$. 
\end{lemma}

\begin{lemma}[Contraction of the Wasserstein distance]
\label{lemma:w}
Let $f:\zzspace\to\yyspace$ and $C > 0$ be a constant such that $\|f\|_L\leq C$. For any distributions $\dist$, $\dist'$ over $\zzspace$, $W_1(f_\sharp\dist, f_\sharp\dist')\leq C\cdot W_1(\dist, \dist')$.
\end{lemma}
\begin{proof}
We use the dual representation of the Wasserstein distance to prove this lemma:
\begin{align*}
    W_1(f_\sharp\dist, f_\sharp\dist') =&~ \sup_{\|f'\|_L\leq 1}\left|\int f'~d(f_\sharp\dist) - \int f'~d(f_\sharp\dist')\right| \\
    =&~ \sup_{\|f'\|_L\leq 1}\left|\int f'\circ f~d\dist - \int f'\circ f~d\dist'\right| \\
    \leq&~ \sup_{\|h\|_L\leq C} \left|\int h~d\dist - \int h~d\dist'\right| \\
    =&~ CW_1(\dist, \dist'),
\end{align*}
where the inequality is due to the fact that for $\|f'\|_L\leq 1$, $\|f'\circ f\|_L\leq \|f'\|_L\cdot \|f\|_L = C$.
\end{proof}
The following fact will also be used in the proof of Theorem~\ref{thm:lowerbound}.
\begin{proposition}
Let $Y$ and $Y'$ be two Bernoulli random variables with distributions $\dist$ and $\dist'$. Then $W_1(\dist, \dist') = |\Pr(Y = 1) - \Pr(Y' = 1)|$.
\end{proposition}
\begin{proof}
Since both $Y$ and $Y'$ are Bernoulli random variables taking values in $\{0, 1\}$, we solve the following linear program to compute $W_1(\dist, \dist')$ according to the primal definition of the Wasserstein distance. Define $D\in\RR^{2\times 2}$ as $D_{ij} = |i - j|$ to be the distance matrix between $Y$ and $Y'$. Then the solution of the following linear program (LP) gives $W_1(\dist, \dist')$:
\begin{equation}
    \begin{aligned}
    & \min_{\gamma\in\RR^{2\times 2}} && \tr(\gamma D) = \sum_{i,j=0}^1 \gamma_{ij}D_{ij}\\
    & \text{subject to} && \sum_{i,j=0}^1 \gamma_{ij} = 1, \gamma_{ij} \geq 0, \sum_{j=0}^1 \gamma_{ij} = \Pr(Y = i), \sum_{i=0}^1 \gamma_{ij} = \Pr(Y' = j).
\end{aligned}
\end{equation}
The objective function $\tr(\gamma D)$ is the transportation cost of a specific coupling $\gamma$, hence the optimal $\gamma^*$ corresponds to the optimal transport between $Y$ and $Y'$. For this simple LP, we have
\begin{align*}
    \sum_{i,j=0}^1 \gamma_{ij}D_{ij} = \gamma_{01} + \gamma_{10}.
\end{align*}
On the other hand, the constraint set gives
\begin{align*}
    \gamma_{00} + \gamma_{01} &= \Pr(Y = 0);\qquad \gamma_{00} + \gamma_{10} = \Pr(Y' = 0); \\
    \gamma_{10} + \gamma_{11} &= \Pr(Y = 1);\qquad \gamma_{01} + \gamma_{11} = \Pr(Y' = 1);
\end{align*}
From which we observe
\begin{equation*}
    |\gamma_{01} - \gamma_{10}| = |\Pr(Y = 1) - \Pr(Y'=1)| = |\Pr(Y = 0) - \Pr(Y' = 0)|,
\end{equation*}
hence,
\begin{align*}
    (\gamma_{01} + \gamma_{10}) + |\gamma_{01} - \gamma_{10}| &= 2 \max\{\gamma_{01}, \gamma_{10}\} \\
                                                              &\geq 2|\gamma_{01} - \gamma_{10}|,
\end{align*}                                
which implies $\forall \gamma$ that is feasible,
\begin{equation*}
    \tr(\gamma D) = \gamma_{01} + \gamma_{10} \geq 2|\gamma_{01} - \gamma_{10}| - |\gamma_{01} - \gamma_{10}| = |\Pr(Y = 1) - \Pr(Y'=1)|.
\end{equation*}
To see that this lower bound is attainable, without loss of generality, assuming that $\Pr(Y = 1) \geq \Pr(Y' = 1)$, the following $\gamma^*$ suffices:
\begin{equation*}
    \gamma^*_{00} = \Pr(Y = 0); \quad\gamma^*_{01} = 0; \quad\gamma^*_{10} = \Pr(Y = 1) - \Pr(Y' = 1); \quad\gamma^*_{11} = \Pr(Y' = 1).\qedhere
\end{equation*}
\end{proof}
With the above tools, we are ready to prove Theorem~\ref{thm:lowerbound}:
\lowerbound*
\begin{proof}
Let $g_\sharp\dist$ be the induced (pushforward) distribution of $\dist$ under the GNN feature encoder $g$. To simplify the notation, we also use $\dist_0$ and $\dist_1$ to denote the conditional distribution of $\dist$ given $A = 0$ and $A = 1$, respectively. Since $h:\zzspace \to \{0, 1\}$ is the task predictor, it follows that $(h\circ g)_\sharp\dist_0$ and $(h\circ g)_\sharp\dist_1$ induce two distributions over $\{0, 1\}$. Recall that $W_1(\cdot, \cdot)$ is a distance metric over the space of probability distributions, by a chain of triangle inequalities, we have:
\begin{align*}
  W_1(\dist(Y\mid A = 0), & \dist(Y\mid A = 1)) \leq W_1(\dist(Y\mid A = 0), (h\circ g)_\sharp\dist_0) \\
  &+ W_1((h\circ g)_\sharp\dist_0, (h\circ g)_\sharp\dist_1) + W_1((h\circ g)_\sharp\dist_1, \dist(Y\mid A = 1)).
\end{align*}
Now by Lemma~\ref{lemma:w}, we have
\begin{equation*}
  W_1((h\circ g)_\sharp\dist_0, (h\circ g)_\sharp\dist_1) \leq C\cdot W_1(g_\sharp\dist_0, g_\sharp\dist_1).
\end{equation*}
Next we bound $W_1(\dist(Y\mid A = a), (h\circ g)_\sharp\dist_a),~\forall a\in\{0, 1\}$:
\begin{align*}
  W_1(\dist(Y\mid A = a), (h\circ g)_\sharp\dist_a) &= |\Pr_\dist(Y = 1\mid A = a) - \Pr_\dist((h\circ g)(X) = 1\mid A = a)| \\ & \hspace{1.5em} \text{(Lemma~\ref{lemma:w}, Both $Y$ and $h(g(X))$ are binary)}\\
  &= |\Exp_\dist[Y\mid A = a] - \Exp_\dist[(h\circ g)(X)\mid A = a]| \\
  &\leq \Exp_\dist[|Y - (h\circ g)(X)|\mid A = a] && \text{(Triangle inequality)} \\
  &= \Pr_\dist(Y \neq (h\circ g)(X)\mid A = a) \\
  &\leq \eps_{Y\mid A= a}(h\circ g),
\end{align*}
where the last inequality is due to the fact that the cross-entropy loss is an upper bound of the 0-1 binary loss. Again, realizing that both $\dist(Y\mid A = 0)$ and $\dist(Y\mid A = 1)$ are Bernoulli distributions, applying Lemma~\ref{lemma:w}, we have
\begin{equation*}
    W_1(\dist(Y\mid A = 0), \dist(Y\mid A = 1)) = \delta_{Y\mid A}.
\end{equation*}
Combining all the inequalities above, we establish the following inequality:
\begin{equation*}
 \eps_{Y\mid A = 0}(h\circ g) + \eps_{Y\mid A = 1}(h\circ g) \geq \delta_{Y\mid A} - C\cdot W_1(g_\sharp\dist_0, g_\sharp\dist_1).
\end{equation*}

For the second part of the inequality, since $\sup_{z\in\zzspace}\|z\|\leq R$, the diameter of $\zzspace$ is bounded by $2R$. Now using the classic result between the TV distance and the Wasserstein distance over a metric space~\citep{gibbs2002choosing}, we have
\begin{equation*}
    W_1(g_\sharp\dist_0, g_\sharp\dist_1) \leq 2R\cdot \dtv(g_\sharp\dist_0, g_\sharp\dist_1),
\end{equation*}
To complete the proof, we show that $\dtv(g_\sharp\dist_0, g_\sharp\dist_1) = \Adv_{\dist}(\FF_A)$: since $\FF_A$ contains all the binary predictors,
\begin{align*}
  \dtv(g_\sharp\dist_0, g_\sharp\dist_1) &= \sup_{E\text{ is measurable}}\left|\Pr_{g_\sharp\dist_0}(E) - \Pr_{g_\sharp\dist_1}(E)\right| \\
  &= \sup_{f_E\in\FF_A}\left|\Pr_{g_\sharp\dist_0}(f_E(Z) = 1) - \Pr_{g_\sharp\dist_1}(f_E(Z) = 1)\right| \\
  &= \sup_{f_E\in\FF_A}\left|\Pr_{g_\sharp\dist}(f_E(Z) = 1\mid A= 0) - \Pr_{g_\sharp\dist}(f_E(Z) = 1\mid A = 1)\right| \\
  &= \Adv_{\dist}(\FF_A),
\end{align*}
where in the second equation above $f_E(\cdot)$ is the characteristic function of the event $E$. Now combining the above two inequalities together, we have:
\begin{align*}
 \eps_{Y\mid A = 0}(h\circ g) + \eps_{Y\mid A = 1}(h\circ g) & \geq \delta_{Y\mid A} - C\cdot W_1(g_\sharp\dist_0, g_\sharp\dist_1) \\
 & \geq \delta_{Y\mid A} - 2RC\cdot \Adv_{\dist}(\FF_A). \qedhere
\end{align*}
\end{proof}

Corollary~\ref{cor:lowerbound} then follows directly from Theorem~\ref{thm:lowerbound}:
\glowerbound*
\begin{proof}
Realize that 
\begin{align*}
  \eps_Y(h\circ g) &= \Pr_\dist(A = 0) \cdot \eps_{Y\mid A = 0}(h\circ g) + \Pr_\dist(A = 1) \cdot \eps_{Y\mid A = 1}(h\circ g) \\
  &\geq \min\{\Pr_\dist(A = 0), \Pr_\dist(A = 1)\}\cdot \left(\eps_{Y\mid A = 0}(h\circ g) + \eps_{Y\mid A = 1}(h\circ g)\right).
\end{align*}
Applying the lower bound in Theorem~\ref{thm:lowerbound} then completes the proof.
\end{proof}

The following lemma about the inverse binary entropy will be useful in the proof of Theorem~\ref{thm:guarantee}:
\begin{lemma}[\citet{calabro2009exponential}]
  Let $H_2^{-1}(s)$ be the inverse binary entropy function for $s\in[0, 1]$, then $H_2^{-1}(s)\geq s / 2\lg (6/s)$.
\label{lemma:inventropy}
\end{lemma}
With the above lemma, we are ready to prove Theorem~\ref{thm:guarantee}.
\guarantee*
\begin{proof}
  To ease the presentation we define $Z = Z^*$. To prove this theorem, let $E$ be the binary random variable that takes value 1 iff $A\neq f(Z)$, i.e., $E = \ind(A\neq f(Z))$. Now consider the joint entropy of $A, f(Z)$ and $E$. On one hand, we have:
  \begin{equation*}
      H(A, f(Z), E) = H(A, f(Z)) + H(E\mid A, f(Z)) = H(A, f(Z)) + 0 = H(A\mid f(Z)) + H(f(Z)).
  \end{equation*}
  Note that the second equation holds because $E$ is a deterministic function of $A$ and $f(Z)$, that is, once $A$ and $f(Z)$ are known, $E$ is also known, hence $H(E\mid A, f(Z))  = 0$. On the other hand, we can also decompose $H(A, f(Z), E)$ as follows:
  \begin{equation*}
      H(A, f(Z), E) = H(E) + H(A\mid E) + H(f(Z)\mid A, E).
  \end{equation*} 
  Combining the above two equalities yields
  \begin{equation*}
      H(E, A\mid f(Z)) = H(A\mid f(Z)).
  \end{equation*}
  On the other hand, we can also decompose $H(E, A\mid f(Z))$ as 
  \begin{equation*}
    H(E, A \mid f(Z)) = H(E\mid f(Z)) + H(A\mid E, f(Z)).
  \end{equation*}
  Furthermore, since conditioning cannot increase entropy, we have $H(E\mid f(Z))\leq H(E)$, which further implies
  \begin{equation*}
      H(A\mid f(Z)) \leq H(E) + H(A\mid E, f(Z)).
  \end{equation*}
  Now consider $H(A\mid E, f(Z))$. Since $A\in\{0, 1\}$, by definition of the conditional entropy, we have:
  \begin{align*}
      H(A\mid E, f(Z)) = \Pr(E = 1) H(A\mid E = 1, f(Z)) + \Pr(E = 0) H(A\mid E = 0, f(Z)) = 0 + 0 = 0.
  \end{align*}
  To lower bound $H(A\mid f(Z))$, realize that 
  \begin{equation*}
      I(A; f(Z)) + H(A\mid f(Z)) = H(A) = I(A; Z) + H(A\mid Z).
  \end{equation*}
  Since $f(Z)$ is a randomized function of $Z$ such that $A\perp f(Z) \mid Z$, due to the celebrated data-processing inequality, we have $I(A;f(Z))\leq I(A; Z)$, which implies
  \begin{equation*}
      H(A\mid f(Z)) \geq H(A\mid Z). 
  \end{equation*}
  Combine everything above, we have the following chain of inequalities hold:
  \begin{equation*}
      H(A\mid Z) \leq H(A\mid f(Z)) \leq H(E) + H(A\mid E, f(Z)) = H(E),
  \end{equation*}
  which implies
  \begin{equation*}
      \Pr(A\neq f(Z)) = \Pr(E = 1) \geq H_2^{-1}(H(A\mid Z)),
  \end{equation*}
  where $H_2^{-1}(\cdot)$ denotes the inverse function of the binary entropy $H(t)\defeq -t\log t - (1 - t)\log(1-t)$ when $t\in[0, 1]$. We then apply Lemma~\ref{lemma:inventropy} to further lower bound the inverse binary entropy function by
  \begin{equation*}
    \Pr(A\neq f(Z))\geq H_2^{-1}(H(A\mid Z)) \geq H(A\mid Z) / 2\lg(6 / H(A\mid Z)),
  \end{equation*}
completing the proof of the first lower bound. For the second part, realize that 
\begin{align*}
    \Pr(f(Z)\neq A) &= \Pr(A = 0)\Pr(f(Z) = 1\mid A = 0) + \Pr(A = 1)\Pr(f(Z) = 0\mid A = 1) \\
                    &\geq \min\{\alpha, 1-\alpha\}\left(\Pr_{\dist_0}(f(Z) = 1) + \Pr_{\dist_1}(f(Z) = 0)\right).
\end{align*}
Now to lower bound $\Pr_{\dist_a}(f(Z) = 1-a)$, we apply the same argument in the proof of Theorem~\ref{thm:lowerbound}, which gives us
  \begin{align*}
  \Pr_{\dist_a}(f(Z) = 1-a) &= \Pr_\dist(f(Z)\neq A\mid A = a) \\
  &= \Exp_\dist[|f(Z) - A|\mid A = a] \\
  &\geq |\Exp_\dist[f(Z)\mid A = a] - \Exp_\dist[A \mid A = a]| \\
  &= |\Pr_\dist(f(Z) = 1\mid A = a) - \Pr_\dist(A = 1\mid A = a)| \\
  &= W_1(\dist_a(f(Z)), A\mid A = a) && \text{(Lemma~\ref{lemma:w}, Both $A$ and $f(Z)$ are binary)}.
\end{align*}
As a last step, using the triangle inequality of $W_1(\cdot, \cdot)$ and Lemma~\ref{lemma:w}, we have
\begin{equation*}
    W_1(\dist_0(f(Z)), A\mid A = 0) + W_1(\dist_1(f(Z)), A\mid A = 1) \geq \delta_{A\mid A} - C W_1^* = 1 - CW_1^*.
\end{equation*}
Combining all the steps above yields
\begin{align*}
    \Pr(f(Z)\neq A) &\geq \min\{\alpha, 1-\alpha\}\left(\Pr_{\dist_0}(f(Z) = 1) + \Pr_{\dist_1}(f(Z) = 0)\right) \\
                    &\geq \min\{\alpha, 1-\alpha\}\left(W_1(\dist_0(f(Z)), A\mid A = 0) + W_1(\dist_1(f(Z)), A\mid A = 1)\right) \\
                    &\geq \min\{\alpha, 1-\alpha\}(1 - CW_1^*),
\end{align*}
which completes the second part of the proof.
\end{proof}

\section{Experimental Setup Details}
\label{sec:setup}

\paragraph{Optimization}
For the objective function, we selected block gradient descent-ascent to optimize our models. In particular, we took advantage of the \textsf{optim} module in PyTorch \citep{pytorch} by designing
a custom gradient-reversal layer, first introduced by~\citep{ganin2016domain}, to be placed between the attacker and the GNN layer we seek to defend. 
The implementation of the graident-reversal layer can be found in the Appendix. During training, we would designate two \textsf{Optimizer} instances, 
one having access to only task-related parameters, and the other having access to attack-related parameters and parameters associated with GNN defense. 
We could then call the \textsf{.step()} method on the optimizers in an alternating fashion to train the entire network, 
where the gradient-reversal layer would carry out both gradient descent (of the attacker) and ascent (of protected layers) as expected. 
Tradeoff control via $\lambda$ is achieved through multiplying the initial learning rate of the adversarial learner by the desired factor.
For graphs that are harder to optimize, we introduce pre-training as the first step in the pipeline, where we train the encoder and the task decoder for a few
epochs before introducing the adversarial learner.

\paragraph{Movielens 1M}

The main dataset of interest for this work is Movielens-1M \footnote{\url{https://grouplens.org/datasets/movielens/1m/}}, a benchmarking dataset in evaluating recommender systems, developed by \citep{ml1m}. 
In this dataset, nodes are either users or movies, and the type of edge represents the rating the user assigns to a movie.
Adapting the formulation of \citep{compfair}, we designate the main task as edge prediction and designate the adversarial task as extracting user-related information
from the GNN embedding using multi-layer perceptrons with LeakyReLU functions \citep{leakyrelu} as nonlinearities. 
Training/test splits are creating using a random 90/10 shuffle.
The network encoder consists of a trainable embedding layer followed by neighborhood aggregation layers.
Node-level embeddings have a dimension of $20$, and the decoder is a naive bilinear decoder, introduced in \citep{gcmc}.
Both the adversarial trainers and the main task predictors are trained with separate Adam optimizers with learning rate set to $0.01$.
Worst-case attackers are trained for $30$ epochs with a batch-size $256$ nodes before the original model is trained for $25$ epochs with a batch-size of $8,192$ edges.

\paragraph{Planetoid}

Planetoid \footnote{Raw data available at \url{https://github.com/kimiyoung/planetoid/tree/master/data}. For this work, we used the wrapper provided by \url{https://pytorch-geometric.readthedocs.io/en/latest/_modules/torch_geometric/datasets/planetoid.html}.} is the common name for three datasets (Cora, CiteSeer, Pubmed) used in benchmarks of graph neural networks in the literature, introduced by \citep{planetoid}.
Nodes in these datasets represent academic publications, and edges represent citation links. Since the Cora dataset is considered to be small to have any practical implications in the performance of our algorithm, we report only the results of CiteSeer and Pubmed. Similar to Movielens, the main task is edge prediction, and the attacker
will seek to predict node attributes from GNN-processed embeddings. The network architecture is message-passing layers connected with ReLU nonlinearities, and
both the decoder and attacker are also single-layer message-passing modules. Regarding training/valid/test splits, we adopt the default split used in the original paper, 
maintained by \citep{torch_geometric}.
The network encoder consists of a trainable embedding layer followed by neighborhood aggregation layers.
Node-level embeddings have a dimension of $64$, and both the adversarial trainers and the main task predictors are trained with separate Adam optimizers with learning rate set to $0.01$.
Worst-case attackers are trained for $80$ epochs with before the original model is trained for $150$ epochs, and the entire graph is fed into the network at once during each epoch.

\paragraph{QM9}

QM9 \footnote{Raw data available at \url{https://s3-us-west-1.amazonaws.com/deepchem.io/datasets/molnet_publish/qm9.zip} and \url{https://ndownloader.figshare.com/files/3195404}} is a dataset used to benchmark machine learning algorithms in quantum chemistry \citep{qm9}, consisting of around $130,000$ molecules
(represented in their spatial information of all component atoms) and $19$ regression targets.
The main task would be to predict the dipole moment $\mu$ for a molecule graph, while the attacker will seek to extract its isotropic polarizability $\alpha$ from the embeddings. 
The encoder is a recurrent architecture consisting of a NNConv \citep{gilmer2017neural} unit, a GRU \citep{gru} unit and a Set2Set \citep{set2set} unit, 
with both the decoder and the attacker (as regressors) 2-layer multi-layer perceptrons with ReLU nonlinearities.
The training/valid/test is selected in the following manner: the order of samples is randomly shuffled at first, then the first 10,000 and 10,000 - 20,000 samples are selected
for testing and validation respectively, and the remaining samples are used for training. 
Preprocessing is done with scripts provided by \citep{torch_geometric} \footnote{Available at \url{https://pytorch-geometric.readthedocs.io/en/latest/_modules/torch_geometric/datasets/qm9.html}}
, using functions from \citep{rdkit}.
Node-level embeddings have a dimension of $64$, and both the adversarial trainers and the main task predictors are trained with separate Adam optimizers with learning rate set to $0.001$.
Worst-case attackers are trained for $30$ epochs with before the original model is trained for $40$ epochs with a batch-size of $128$ molecular graphs.

\paragraph{FB15k-237/WN18RR}

These two datasets are benchmarks for knowledge base completion: while FB15k-237 \footnote{\url{https://www.microsoft.com/en-us/download/details.aspx?id=52312}} is semi-synthetic with nodes as common entities, 
WN18RR \footnote{\url{https://github.com/TimDettmers/ConvE}} is made by words found in the thesaurus. 
Our formulation is as follows: while the main task from both datasets is edge prediction, the attackers' goals are different:

\begin{itemize}
  \item For FB15k-237, we took node-level attributes from \citep{moon} \footnote{\url{https://github.com/cmoon2/knowledge_graph}}, and task the attacker with predicting the 50-most frequent labels. Since a node in FB15k-237 may have 
  multiple labels associated with it, adversarial defense on this may be seen as protecting sets of node-level attributes, in contrast to single-attribute defense in other experimental
  settings.
  \item For WN18RR, we consider two attributes for a node (as a word): its word sense (sense greater than 20 are considered as the same heterogeneous class), and part-of-speech tag. 
  The labels are obtained from \citep{wn18rr-mlj13} \footnote{\url{https://everest.hds.utc.fr/doku.php?id=en:smemlj12}}.
\end{itemize}

As for the architecture, we used a modified version of the CompGCN paper \citep{compgcn}, where the attacker has access to the output of the CompGCN layer (of dimension $200$), 
and the original task utilizes the ConvE model for the decoder. 
The training/valid/test split also aligns with the one used in the CompGCN paper. On both datasets, the adversarial trainers and main task predictors are trained with separate Adam optimizers with learning rate set to $0.001$.
Worst-case attackers are trained for $30$ epochs with a batch-size of $128$ nodes before the original model is trained for $120$ epochs after $35$ epochs of pre-training, with a batch-size of $128$ nodes.

\paragraph{Computing Infrastructure}

All models are trained with NVIDIA GeForce\textregistered \text{ } RTX 2080 Ti graphics processing units (GPU) with 11.0 GB GDDR6 memory on each card, and non-training-related operations are performed
using Intel\textregistered \text{ } Xeon\textregistered \text{ } Processor E5-2670 (20M Cache, 2.60 GHz).

\newpage

\paragraph{Estimated Average Runtime}
Below are the averge training time per epoch for each models used in the main text, when the training is performed on the computing infrastructure mentioned above:

\begin{table}[htb]
  \centering
  \label{runtime}
  \begin{tabular}{l l l}
    \toprule
    \textsc{Dataset}      & Encoder & $t$ \\ [0.5ex] 
    \midrule
    \textsc{CiteSeer}     & ChebNet & $0.0232$s  \\
                          & GCN & $0.0149$s  \\
                          & GAT & $0.0282$s  \\
    \textsc{Pubmed}       & ChebNet & $0.0920$s  \\
                          & GCN & $0.0824$s  \\
                          & GAT & $0.129$s  \\
    \textsc{QM9}          & MPNN & $199.25$s  \\
    \textsc{Movielens-1M} 
                          & GCN & $12.05$s  \\
                          & GAT & $45.86$s  \\
    \textsc{FB15k-237}    & CompGCN & $463.39$s  \\
    \textsc{WN18RR}       & CompGCN & $181.55$s  \\
    \bottomrule
   \end{tabular}
   \vspace*{0em}
\end{table}

\section{Degredation of RMSE on Movielens-1M dataset Regarding Neighborhood Attack}

This is a supplementary figure for the neighborhood attack experiments introduced in the main section. Band represents 95\% confidence interval over five runs.
\begin{figure}[H]
  \centering
  \includegraphics[width=0.5\linewidth]{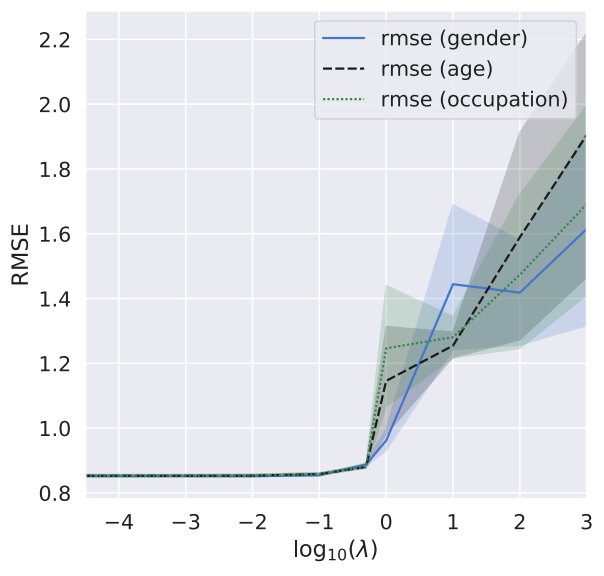}
  \label{neighborrmse}
\end{figure}

\newpage
\section{N-Hop Algorithm for Neighborhood Defense}
Intuitively, this algorithm greedily constructs a path of length $n$ by uniformly picking a neighbor from the current end of the path and checking if the node
has existed previously in the path, avoiding formation of cycles. Worst-case running time of this algorithm is $O(n^2)$, because in each step of the main loop,
the algorithm performs $O(n)$ checks in the worst case scenario.
\begin{algorithm}[htb]
    \label{alg:nhop}
    \caption{Monte-Carlo Probabilistic N-Hop}
    \begin{algorithmic}
    \STATE  {\bfseries Input: } {$G=(V,E)$: undirected graph (via adjacency list); $v \in V$: starting node; $n \geq 1$: hop}
    \STATE  {\bfseries Output: } {On success: $v' \in V$ such that $d(v, v') = n$ or $\textsc{no}$ if such vertex doesn't exist;
              On failure: $v' \in V$ such that $1 \leq d(v, v') \leq n$ 
                             or $\textsc{no}$ if such vertex doesn't exist}
    \STATE $V = \emptyset$ \COMMENT{Initial path is empty}
    \STATE $t = 0$ 
    \STATE $v' = v$
  
     \REPEAT
      \STATE $S = [  \mathcal{N}(v') ]$ \COMMENT{$O(1)$ time by adjacency list}
      \STATE $i = \text{RandInt}(0, |S|)$ \COMMENT{$O(1)$ uniform random sample (without replacement)}
      \STATE $e = S\text{.pop}(i)$
      \REPEAT
        \STATE {$i = \text{RandInt}(0, |S|)$}
        \STATE {$e = S\text{.pop}(i)$}
      \UNTIL {$\neg(e \in V$ \text{and} $S \neq [])$} \COMMENT{Loop runs at most $O(n)$ times}
  
      \IF{$e \notin V$}   
        \STATE $V = V \cap \{ e\}$
        \STATE $v' = e$ 
      \ELSE 
       \STATE reject with \textsc{no} \COMMENT{Current path not satisfiable, reject}
      \ENDIF
      \STATE $t = t+1$
     \UNTIL $t >= n$
      \STATE accept with $v'$
    \end{algorithmic}
  \end{algorithm}